\documentclass[letterpaper, 10 pt, conference]{ieeeconf}
\IEEEoverridecommandlockouts
\usepackage{amsmath,amssymb,amsfonts}
\usepackage{algorithmic}
\usepackage{graphicx}
\usepackage{textcomp}
\usepackage{xcolor}
\usepackage{macros}
\usepackage{glossaries}
\usepackage{xspace}

\usepackage{amsthm}
\usepackage{dsfont}
\usepackage{bm}
\usepackage{booktabs}
\usepackage[colorlinks=false]{hyperref}
\usepackage[capitalize]{cleveref}

\newtheoremstyle{boldtitle}% name of the style to be used
  {\topsep}% measure of space to leave above the theorem. E.g.: 3pt
  {\topsep}% measure of space to leave below the theorem. E.g.: 3pt
  {\itshape}% name of font to use in the body of the theorem
  {0pt}% measure of space to indent
  {\bfseries}% name of head font
  {.}% punctuation between head and body
  {5pt plus 1pt minus 1pt}% space after theorem head; " " = normal interword space
  {\thmname{#1}\thmnumber{ #2}\thmnote{ (\textbf{#3})}}% Manually bold the note (#3)

\newtheorem{proposition}{Proposition}

\newcommand{\reach}{\textsf{Reach}}
\newcommand{\retreat}{\textsf{Retreat}}
\newcommand{\slideLeft}{\textsf{Slide Left}}
\newcommand{\slideRight}{\textsf{Slide Right}}
\newcommand{\press}{\textsf{Press}}
\newcommand{\circular}{\textsf{Circle}}
\newcommand{\wipe}{\textsf{Wipe}}

\begin{document}

\title{\LARGE \bf
Imitation Learning via Discrete and Rhythmic Primitive Composition
}

\author{Rafid Saif$^{1, 2}$, Auke Ijspeert$^{3}$, and Hsiu-Chin Lin$^{1, 2}$%
\thanks{This work has been submitted to the IEEE for possible publication. Copyright may be transferred without notice, after which this version may no longer be accessible.}
\thanks{We acknowledge the support of the Natural Sciences and Engineering Research Council of Canada (NSERC).}% <-this % stops a space
\thanks{$^{1, 2}$McGill University, and Mila - Quebec AI Institute, Canada. }
\thanks{$^{3}$\'Ecole Polytechnique F\'ed\'erale de Lausanne (EPFL), Switzerland. }%
\thanks{Correspondence: \tt\small rafid.saif@mail.mcgill.ca}
}

\maketitle
\thispagestyle{empty}
\pagestyle{empty}

\begin{abstract}
Long-horizon robotic manipulation is inherently composed of semantically meaningful sub-tasks. While Movement Primitives (MPs) can model these components, their composition typically relies on manual engineering or sample-inefficient reinforcement learning. Existing supervised methods are limited by the lack of MP labels in standard datasets and often fail to generalize across discrete and rhythmic movements. We propose Neural Dynamic Composition (NDC), a supervised framework to compose Dynamic Movement Primitives (DMPs) using only raw trajectories for supervision, given a single example of each primitive. Our approach exploits parallel combinations of DMPs as a differentiable relaxation of categorical primitive selection, and utilizes a warp-invariant loss formulation to identify primitives despite spatial or temporal variations. A staged training curriculum is employed to learn the selection policy and DMP parameters, unifying discrete and rhythmic movements in one framework. Evaluated on 2D handwriting and 3D robotic tasks, our method outperforms baselines on most tasks and metrics. An ablation study demonstrates how each design choice contributes to the system's performance.
\end{abstract}


\newacronym{dmp}{DMP}{Dynamic Movement Primitive}
\newacronym{mp}{MP}{Movement Primitive}
\newacronym{ic}{IC}{initial condition}
\newacronym{dtw}{DTW}{Dynamic Time Warping}
\newacronym{sdtw}{SDTW}{Soft Dynamic Time Warping}
\newacronym{ndmp}{NDMP}{Neural Dynamic Movement Primitive}
\newacronym{mse}{MSE}{mean squared error}
\newacronym{rl}{RL}{reinforcement learning}
\newacronym{il}{IL}{imitation learning}
\newacronym{ml}{ML}{machine learning}
\newacronym{bc}{BC}{behavior cloning}
\newacronym{gcbc}{GCBC}{goal-conditioned behavior cloning}
\newacronym{rqa}{RQA}{recurrence quantification analysis}
\newacronym{lfd}{LfD}{learning from demonstrations}
\newacronym{ood}{OOD}{out-of-distribution}
\newacronym{mlp}{MLP}{multi-layer perceptron}
\newglossaryentry{ndc}{
  name={NDC},
  description={Neural Dynamic Composition}
}
\newacronym{ldlj}{LDLJ}{Log Dimensionless Jerk}
\newacronym{sparc}{SPARC}{Spectral Arc Length}

\section{Introduction}
\noindent
When humans learn to replicate a task from a demonstration, we typically reason in terms of meaningful sub-tasks, or \emph{skills}. For example, when observing someone making a cup of tea, we naturally segment the task into a sequence of distinct skills, such as \emph{reaching} for the pot, \emph{lifting} the lid, or \emph{stirring}. Through this segmentation, we also implicitly infer which skill is \emph{discrete} (one-off) or \emph{rhythmic} (cyclic repetitions). We can replicate the task easily by recomposing these skills with our understanding of each skill’s objective.

\begin{figure}[t!]
    \centering
    \includegraphics[width=\linewidth]{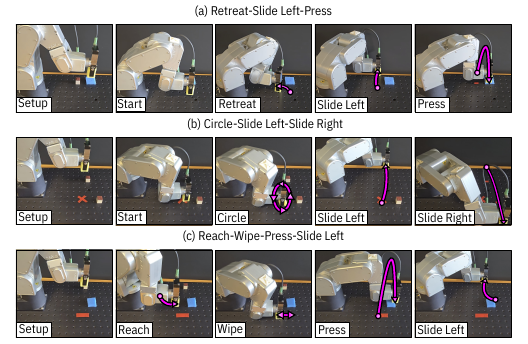}
    \caption{{Examples of sequential tasks with a mixture of discrete and rhythmic movements.}}
    \label{fig:ndc_robot_rollouts}
\end{figure}

Skill-based learning brings this task-aware action execution to robotics. However, existing approaches face a major hurdle, requiring manual skill annotations \cite{Sun2025ARCHHierarchicalHybrida, Jaquier2022LearningSequenceBlend} or resorting to \gls{rl} \cite{Chen2024SCaRRefiningSkill, Shankar2020DiscoveringMotorPrograms, Nasiriany2022AugmentingReinforcementLearning} since explicit skill selection is not differentiable. The first approach requires laborious annotation of task demonstrations; the latter is sample inefficient. Differentiable, unsupervised alternatives exist \cite{Tanneberg2021SKIDRAWSkill}, but these rely on latent representations of skills, forgoing the interpretability of pre-defined skills. Notably, these approaches do not address rhythmic skills.

On the other hand, \glspl{mp} \cite{Ijspeert2013DynamicalMovementPrimitivesa, Li2023ProDMPUnifiedPerspectivea, Huang2019KernelizedMovementPrimitivesa, Kulak2020FourierMovementPrimitivesa} represent a mathematical framework to define both discrete and rhythmic skills. \Glspl{mp} offer further advantages such as compact representation, temporal/spatial scaling, and goal-reaching guarantees. However, the selection and sequencing of primitives and their parameters is generally guided by hand-crafted heuristics, limiting their deployment in complex tasks~\cite{Nah2025CombiningMovementPrimitivesa}. Prior work has studied selecting parameters for known primitive sequences, yet leaves primitive selection unaddressed~\cite{Wu2024OneShotRobustImitation}. A method is needed to predict primitive sequences and parameters directly from unstructured, unsegmented demonstrations.

In this paper, we address this problem by proposing \oursfirst (\ours), a framework that selects primitives in sequence and predicts their execution parameters in a supervised manner \emph{without primitive annotations}, even in the presence of rhythmic task segments.
We argue that the inherent structure of unsegmented demonstration trajectories can provide a sufficient supervisory signal for this problem.
\ours operates in a challenging ``weakly supervised'' setting: given a library of primitives (e.g., reaching, wiping, or circular motions) and singular demonstrations for each primitive, we train using raw demonstration trajectories, without explicit primitive labels. 
We empirically validate \ours across 2D and 3D domains, including real-robot deployment (see Fig.~\ref{fig:ndc_robot_rollouts}), showing \ours generalizes across primitive compositions and parameterizations. Our contributions are:
\begin{itemize}
\item A unified high-level policy that infers sequences of \emph{discrete and rhythmic} primitives, with their execution parameters, from unsegmented, unlabeled trajectories.
\item A theoretically motivated training objective that acts as a proxy for primitive-level supervision despite spatial or temporal variations.
\end{itemize}

\section{Related Work}
\label{sec:related_work}

\noindent
\textbf{Movement Primitives.} \Glspl{dmp}~\cite{Ijspeert2013DynamicalMovementPrimitivesa} were introduced as a representation for stable dynamical systems for \gls{il}. Further work extended this to probabilistic formulations~\cite{Paraschos2013ProbabilisticMovementPrimitivesa, Li2023ProDMPUnifiedPerspectivea, Kulak2020FourierMovementPrimitivesa}. Recently, \cite{Nah2025CombiningMovementPrimitivesa} showed that \glspl{dmp} can be combined using the property of contraction. Despite the useful representation properties of \glspl{mp}, they are typically sequenced and parametrized by hand. 

\textbf{Primitive Composition.} The first group of skill-based methods predicts parameters for known sequences of primitives~\cite{Wu2024OneShotRobustImitation, Mishra2023GenerativeSkillChaining, Shiarlis2018TACOLearningTask}.  MiLa~\cite{Wu2024OneShotRobustImitation} uses \glspl{dmp} in a fixed, known sequence, leaving \gls{dmp} selection as an open problem. A second line of work addresses this using \gls{rl}~\cite{Chen2024SCaRRefiningSkill, Nasiriany2022AugmentingReinforcementLearning} or known primitive labels~\cite{Jaquier2022LearningSequenceBlend, Sun2025ARCHHierarchicalHybrida}. PRIME~\cite{Gao2024PRIMEScaffoldingManipulation} extracts primitive labels autonomously, but requires environment exploration. A third category involves discovering skills and compositions~\cite{Shankar2020DiscoveringMotorPrograms, Shankar2020LearningRobotSkills, Tanneberg2021SKIDRAWSkill}. While this approach requires the least intervention, it lacks the interpretability of pre-specified skills.

\textbf{Rhythmic Task Learning.} Learning rhythmic tasks is often considered as a specialized problem in robot learning, separated from discrete movements, due to difference in the representations~\cite{Yang2022LearningPeriodicTasksa}. 
Recent works focus on enforcing periodicity in latent network representations~\cite{Li2024FLDFourierLatent} or periodic skill discovery with a similar approach~\cite{Park2026PeriodicSkillDiscovery}, primarily using \gls{rl}. ViPTL~\cite{Yang2022LearningPeriodicTasksa} learns periodic tasks from human demonstrations; however, it is limited to purely periodic tasks.

\section{Background}
\subsection{Dynamic Movement Primitives}
\label{subsec:dmp}
\noindent 
\Glspl{dmp} \cite{Ijspeert2013DynamicalMovementPrimitivesa} are dynamical systems that replicate discrete and rhythmic demonstration trajectories. The \emph{canonical system} drives the timing, $\phase(t)$:
\begin{align}
    \text{discrete: }& \tau \dotphased(t) = -\alpha_\phase \phased(t) \\
    \text{rhythmic: }& \tau \dotphaser(t) = 1,
\end{align}
where $\tau$ is a time scale and $\alpha_\phase$ is the discrete phase decay rate.
A set of $\nbasis$ basis functions $\vbasisd, \vbasisr: \realp \mapsto \realp^\nbasis$ maps the phase to basis activations, using Gaussian (discrete) and von Mises (rhythmic) functions:
\begin{align}
    \text{discrete: }& \basisd_i = \exp \big(-h_i(\phased - c_i)^2 \big), \\
    \text{rhythmic: }&\basisr_i = \exp \Bigl(h_i \bigl(\cos(\phaser - c_i)-1 \bigr) \Bigr),
\end{align}
with centers $\{c_i\}_{i=1}^{\nbasis}$ and widths $\{h_i\}_{i=1}^{\nbasis}$ taken from \cite{Nah2025CombiningMovementPrimitivesa}. These activations form the \emph{forcing function} $\vf_t: \realp^\nbasis \mapsto \real^\nstate$, which determines the transient \emph{shape} of the trajectory:
\begin{align}
    \text{discrete: }& \vf_t(\vbasisd) = \phased\frac{\mweight\vbasisd(\phased)}{\trans{\vone}\vbasisd(\phased)}, \\
    \text{rhythmic: }& \vf_t(\vbasisr) = \frac{\mweight\vbasisr(\phaser)}{\trans{\vone}\vbasisr(\phaser)},
\end{align}
where $\mweight \in \real^{\nstate \times \nbasis}$ is a basis function weight matrix.
This forcing term is added to a point-attractor to form the \emph{transformation system}:
\begin{equation}
    \tau^2 \ddot{\vy}_t = \alpha_z \bigl(\beta_z (\vg-\vy_t) - \tau\dot{\vy}_t\bigr) \ + \forcefunc_t,
    \label{eq:dmp_transformation_system}
\end{equation}
where $\vy_t, \dot{\vy}_t \in \real^\nstate$ are the state and velocity, $\alpha_z, \beta_z \in \realp$ control the attractor dynamics, and $\vg \in \real^\nstate$ is the \gls{dmp} goal. The forcing term's argument is dropped for ease of notation. A \gls{dmp} can be spatially scaled by $\ks \in \realp$ by setting $\bar{\vf}_t = \ks \eye{\nstate} \vf_t$, and $\bar{\vg} = \ks \eye{\nstate} (\vg - \vy_0) + \vy_0$, with $\eye{\nstate} \in \real^{\nstate \times \nstate}$ being identity. The scaled system state and derivatives are:
\begin{align}
    \bar{\vy}_t &= \ks \eye{\nstate} (\vy_t - \vy_0) + \vy_0, \label{eq:scaled_pos} \\
    \dot{\bar{\vy}}'_t &= \ks \eye{\nstate} \dot{\vy}_t, \label{eq:scaled_vel} \\
    \ddot{\bar{\vy}}'_t &= \ks \eye{\nstate} \ddot{\vy}_t. \label{eq:scaled_acc}
\end{align}

\subsection{Combining Movement Primitives}
\label{subsec:combining_dmps}
\noindent
Prior work proposed a method to combine \glspl{dmp} in parallel, resulting in a system that exhibits an intermediate behavior between different primitives~\cite{Nah2025CombiningMovementPrimitivesa}. Rearranging \cref{eq:dmp_transformation_system} yields:
\begin{equation}
    \tau^2 \ddot{\vy}_t + \alpha_z \tau \dot{\vy}_t + \alpha_z \beta_z \vy_t  = \alpha_z \beta_z \vg + \vf_t.
    \label{eq:dmp_rearranged}
\end{equation}
The expression on the right-hand side is defined as the primitive input $\vu$ of a \gls{dmp}:
\begin{equation}
    \vu = \alpha_z \beta_z \vg + \vf_t.
    \label{dmp_primitive_input}
\end{equation}
$\nlibrary$ \glspl{dmp} can then be combined via a convex combination:
\begin{equation}
    \vu = \sum_{\ilibrary=1}^\nlibrary \pi_\ilibrary \vu_\ilibrary,\quad \sum_{\ilibrary=1}^\nlibrary \pi_\ilibrary = 1.
    \label{eq:dmp_parallel_combination}
\end{equation}
The combined system retains the stability properties of \glspl{dmp}. In \cref{eq:dmp_parallel_combination}, the weights $\{\pi_\ilibrary\}$ naturally allow a probabilistic interpretation: a higher probability for a given \gls{dmp} proportionally increases its influence on the generated trajectory. 
This is used in \ours to predict the combined effect of all the library primitives, and incentivize those which better match the demonstrations.

\subsection{Neural Dynamic Movement Primitives}
\label{subsec:ndmp}
\noindent
Previous work \cite{Bahl2020NeuralDynamicPoliciesa} showed that a \gls{dmp} can be embedded as a differentiable neural network layer. The layers preceding the \gls{dmp} predict its defining parameters: goal $\vg$ and forcing function weight $\mweight$.
The transformation system is simulated using numerical integration with a step size of $\delta t$:
\begin{align}
    \label{eq:ndp_integration}
    \dot{\hat{\vy}}_{t+1} &= \dot{\hat{\vy}}_t + \ddot{\hat{\vy}}_t \delta t \nonumber \\
    \hat{\vy}_{t+1} &= \hat{\vy}_t + \dot{\hat{\vy}}_t \delta t.
\end{align}
By iteratively unrolling from a known \gls{ic}, the integrator produces a trajectory $\hat{Y} = [\hat{\vy}_0, \ldots, \hat{\vy}_T]$ that is differentiable with respect to the predicted $\vg$ and $\mweight$. 

This differentiability enables end-to-end training: in the original formulation, the predicted $\hat{Y}$ is compared to the expert demonstration trajectory, $Y,$ using a pointwise \gls{mse} loss. We call this formulation \baselinemse. While straightforward, this point-wise loss enforces strict temporal alignment which can be problematic as expert demonstrations may involve speed differences.

\begin{figure*}[t]
    \centering
    \includegraphics[width=0.8\linewidth]{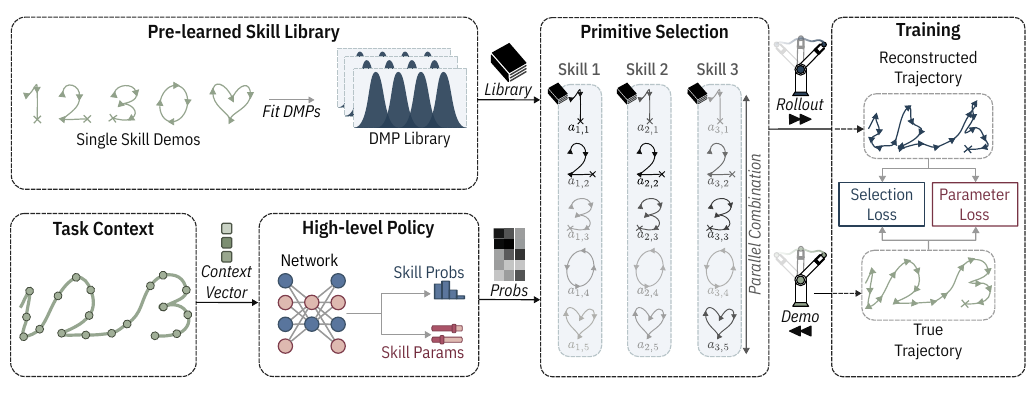}
    \caption{Overview of NDC. NDC leverages a pre-built skill library, using DMPs fitted to a single segmented skill demonstration. Conditioned on a task context, it predicts probabilities for each skill in the library across the sequence, and associated primitive parameters. The probabilities are used to combine DMPs, and rolled out into a resulting trajectory. By computing two losses with respect to the demonstration trajectory, the network is trained in a fully supervised manner.}
    \label{fig:ndc_overview}
\end{figure*}

\section{Neural Dynamic Composition}
\label{chap:method}

\noindent
We propose \ours, a hierarchical approach for learning from a mixture of discrete and rhythmic movements. We provide the problem formulation for \ours in \cref{subsec:ndc_problem_formulation}, detail our differentiable primitive selection approach in \cref{sec:diff_primitive_sel}, followed by suitable training objectives in \cref{sec:train_objective}. Then, we extend the formulation to rhythmic tasks in \cref{sec:ext_periodic} and conclude with practical guidelines in \cref{sec:practical_implementation} (See Fig.~\ref{fig:ndc_overview}).

\subsection{Problem Formulation}
\label{subsec:ndc_problem_formulation}
\noindent
Consider, in a multi-task setting, a dataset $\{(\stateall_n, \obsall_n)\}_{n=1}^N$, where each element represents a task demonstration of $T$ time-steps. Let $\stateall_n = \{\vy_t\}_{t=1}^T$, where $\vy_t \in \real^\nstate$ is the $\nstate$-dimensional robot position at time $t$, and $\obsall_n$ is a \emph{context encoding}. This encodes what the task is (i.e., which sub-tasks are used, and relevant sub-goals), and may comprise images or state representations. In this work, downsampled expert trajectories represent the task context as they encode sub-tasks as well as relevant parameters, while  retaining temporal characteristics required to represent rhythmic movements.

We assume that each task is composed of primitives $\skill$, drawn from some library $\library = \{\skill_\ilibrary\}_{\ilibrary=1}^\nlibrary$. Notably, trajectory segments may vary in length and execution speed, even for the same primitive. We do not assume access to ground-truth primitive labels for the tasks; however, we assume one example of each primitive is given to fit our library.

Let $\trajgen(\skill, \ks, \vg)$ denote the trajectory segment generated by a library primitive with a spatial scaling factor $\ks$ and a goal $\vg$ using \cref{eq:ndp_integration}. For a task composed of $\nskill$ primitives, the state trajectory can be expressed as:
\begin{equation}
    \hat{\stateall} = \trajgen(\skill^{(1)}, \ks^{(1)}, \vg^{(1)}) \oplus  \ldots \oplus \trajgen(\skill^{(\nskill)}, \ks^{(\nskill)}, \vg^{(\nskill)}),
    \label{eq:skill_concat_def}
\end{equation}
where $\oplus$ denotes the concatenation operator. Let $f_\theta$, parameterized by $\theta$, be a function that predicts the sequence of primitives and their parameters:
\begin{equation}
    \{\skill^{(\iskill)}, \ks^{(\iskill)}, \vg^{(\iskill)} \}_{\iskill=1}^\nskill = f_\theta(\obsall).
    \label{eq:obs_conditioned_acts}
\end{equation}
Our goal is to predict the correct sequence of primitives and their parameters to reproduce demonstrations.
% The optimization objective is therefore:
% \begin{equation}
%     \theta^* = \arg \min_{\theta} \sum_{i} \mathcal{D} \left(\hat{\stateall}_i, \stateall_i \right)
%     \label{eq:objective_func}
% \end{equation}
% where $\mathcal{D}$ is a loss function between two trajectories.

% \subsection{Library Formation}
% \label{subsec:sel_primitive}

\subsection{Differentiable Primitive Combination}
\label{sec:diff_primitive_sel}

\noindent
\ours operates in a supervised setting, making categorical predictions over the skill library as well as their continuous parameters. However, the ground truth skills and parameters are unobserved in the dataset which consists solely of unsegmented demonstration trajectories. To enable this, we consider the loss in \emph{trajectory} space (similar to \cref{subsec:ndmp}) with respect to the demonstration trajectory. The primary requirement for the selection process is differentiability of the predicted trajectory with respect to network predictions, with the possibility of both discrete and rhythmic movements. 

We start with  the \gls{dmp} library formulation. Given one example trajectory for each of $\nlibrary$ primitives, we fit a \gls{dmp} using linear regression to form the library, $\library = \{\skill_\ilibrary\}_{\ilibrary=1}^\nlibrary$, where $\skill_\ilibrary$ is a 3-tuple characterizing a \gls{dmp}, $(\vg_\ilibrary,\, \mweight_\ilibrary,\, r_\ilibrary)$. The first two elements are the \gls{dmp} nominal goal and forcing function weights, and $r_\ilibrary$ is a boolean indicating whether a \gls{dmp} is rhythmic.

We exploit parallel DMP combinations from \cref{subsec:combining_dmps} as a differentiable relaxation of primitive selection: one-hot probabilities recover discrete selection, while soft probabilities allow gradients to propagate through the DMP rollout. We predict probabilities over the library, combine forcing functions as in \cref{eq:dmp_parallel_combination}, and through a single integration, obtain a trajectory representing the combined effect of all library primitives weighted by their probabilities (see \cref{fig:diff_selection_example}). At each integration step, the forcing function is evaluated as:
\begin{equation}
    \vf_t = \sum_{\ilibrary=1}^\nlibrary (1 - \indicator_{r_\ilibrary})\pi_\ilibrary \phased \frac{\mweight_\ilibrary \vbasisd}{\trans{\vone}\vbasisd} + \indicator_{r_\ilibrary}\pi_\ilibrary \frac{\mweight_\ilibrary \vbasisr}{\trans{\vone} \vbasisr},
\end{equation}
where $\indicator_{r_\ilibrary}$ is $1$ for rhythmic primitives and $0$ otherwise. Next, we formulate a training objective to encourage both primitive selection and parameter regression.

\begin{figure}
    \centering
    % \vspace{4pt}
    \includegraphics[width=0.95\linewidth]{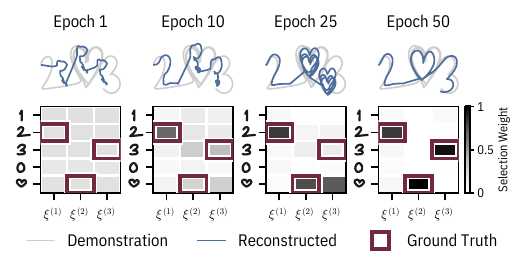}
    \caption{Evolution of combined trajectories (top) and predicted probabilities (bottom) for an example with fixed parameters.}
    \label{fig:diff_selection_example}
\end{figure}

\subsection{Training Objectives}
\label{sec:train_objective}
\noindent
In real-world demonstrations,  a skill can be performed with different parameters, such as different goal positions or speeds. \Glspl{dmp} can capture some key variations by modulating goal positions, spatial scaling and time scaling (discussed in \cref{subsec:dmp}). This imposes two requirements on \ours: first, our training process must allow primitive identification despite these variations, and second, given some choice of primitive, select the parameters to reproduce demonstrations.

\subsubsection{Primitive Selection Objective}
\label{subsec:sel_varying_goals}
Because we consider unsegmented trajectory demonstrations, we additionally have no knowledge of \emph{where} each primitive occurs: its start/end time is unknown, in addition to the goal and scaling parameters discussed above. Formally, let $\{\vg^{(\iskill)}\}_{\iskill=1}^{\nskill}$ and $\{\ks^{(\iskill)}\}_{\iskill=1}^{\nskill}$ denote these unknown goals and scaling. Executing the same sequence of \glspl{dmp} with incorrect parameters results in an entirely different position trajectory. This precludes a position based loss as in \cref{subsec:ndmp} as a stable training signal for primitive selection.

Instead, the cosine similarity of velocities/accelerations is resistant to parameter errors. Let $\bar{\vy}_t$ denote the trajectory under ``perfect'' scaling by $\ks$ (IC $\vy_0$, goal $\ks\eye{\nstate}(\vg-\vy_0)+\vy_0$). \Cref{eq:scaled_vel,eq:scaled_acc} show these quantities retain the same direction as the nominal trajectory despite scaling, so
\begin{equation}
    \operatorname{sim}(\dot{\bar{\vy}}_t, \dot{\vy}_t) = 1, \quad \operatorname{sim}(\ddot{\bar{\vy}}_t, \ddot{\vy}_t) = 1.
    \label{eq:perfect_scaling_sim}
\end{equation}
Next, we show that despite differences in \gls{ic} and goal position, these quantities converge to $1$.

\begin{proposition}[Shift and Scale Invariance]
\label{proposition:dmp_shift_scale_invariance}
For a scaled \gls{dmp} with arbitrary \gls{ic} $\vy_0' = \vy_0 + \Delta_1$ and goal $\vg' = \ks \eye{\nstate} (\vg - \vy_0) + \Delta_2$, the cosine similarity of the velocity $\dot{\vy}_t'$ and acceleration $\ddot{\vy}_t'$ relative to the nominal trajectory converges to 1:
\begin{equation}
\operatorname{sim}(\dot{\vy}_t', \dot{\vy}_t) \to 1, \quad \operatorname{sim}(\ddot{\vy}_t', \ddot{\vy}_t) \to 1 \quad \text{as} \quad t \to \infty.
\label{eq:dmp_sim_convergence}
\end{equation}
\end{proposition}

\begin{proof}[Proof (Sketch)]
Let $\vy_t''$ denote the trajectory obtained by uniformly shifting $\vy_t'$ by $-\Delta_2$, i.e., with \gls{ic} $\vy_0'' = \vy_0' - \Delta_2 = \vy_0 + \Delta_1 - \Delta_2$ and goal $\vg'' = \vg' - \Delta_2 = \ks\eye{\nstate}(\vg-\vy_0)$. A uniform shift does not change velocity or acceleration, so $\dot{\vy}_t' = \dot{\vy}_t''$ and $\ddot{\vy}_t' = \ddot{\vy}_t''$.

Since $\vy_t''$ and $\bar{\vy}_t$ share the same goal $\ks\eye{\nstate}(\vg-\vy_0)$ and differ only in \gls{ic}, the contraction property of \glspl{dmp}~\cite{Nah2025CombiningMovementPrimitivesa, Perk2006MotionPrimitivesRobotica} implies $\dot{\vy}_t'' \to \dot{\bar{\vy}}_t$ and $\ddot{\vy}_t'' \to \ddot{\bar{\vy}}_t$ as $t \to \infty$. From \cref{eq:perfect_scaling_sim},
\begin{equation*}
    \operatorname{sim}(\dot{\vy}_t', \dot{\vy}_t) = \operatorname{sim}(\dot{\vy}_t'', \dot{\vy}_t) \to \operatorname{sim}(\dot{\bar{\vy}}_t, \dot{\vy}_t) = 1,
\end{equation*}
and analogously for $\ddot{\vy}_t'$.
\end{proof}
\begin{figure}
    \centering
    % \vspace{4pt}
    \includegraphics[width=0.85\linewidth]{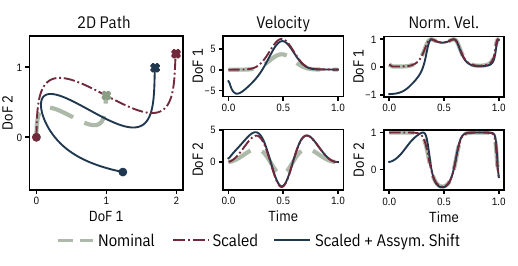}
    \label{fig:placeholder}
    \caption{DMP rollouts with varying spatial scale and boundary conditions. Raw velocities converge to scaled nominal velocities. Meanwhile, velocities normalized along the movement dimension converge to the nominal normalized velocity, even over the finite horizons used in training.}
\label{fig:diff_ks_g_ic}
\end{figure}

An illustrative example of \cref{proposition:dmp_shift_scale_invariance} is shown in \cref{fig:diff_ks_g_ic}. To make this objective resilient to speed differences in demonstration trajectories, we employ the \gls{sdtw} algorithm~\cite{Cuturi2017SoftDTWDifferentiableLoss} using the implementation from~\cite{Maghoumi2021DeepNAGDeepNonAdversarial}. Unlike \gls{mse}, which assumes one-to-one temporal correspondence, \gls{sdtw} finds an optimal alignment $\mathcal{W}$ between two time-series based on some cost function. We utilize a cosine similarity metric based on \cref{eq:dmp_sim_convergence} as the ground cost function between each pair of points in the predicted and reference trajectories:
\begin{multline}
    c(\dot{\hat{\vy}}_i, \dot{\vy}_j, \ddot{\hat{\vy}}_i, \ddot{\vy}_j) = \lambda_{\text{vel}} \bigl(1 - \operatorname{sim}(\dot{\hat{\vy}}_i, \dot{\vy}_j)\bigr) \\ + \lambda_{\text{acc}} \bigl( 1 - \operatorname{sim}(\ddot{\hat{\vy}}_i, \ddot{\vy}_j) \bigr),
\end{multline}
where $\lambda_{\text{vel}}, \lambda_{\text{acc}} \in \realp$ are weighting constants. We use $\lambda_{\text{vel}} = 0.65,\ \lambda_{\text{acc}} = 0.35$. We denote this loss $\lsel$:
\begin{equation}
    \lsel(\hat{Y}, Y) = \min_{\mathcal{W}} \sum_{i,j\in \mathcal{W}}  c(\dot{\hat{\vy}}_i, \dot{\vy}_j, \ddot{\hat{\vy}}_i, \ddot{\vy}_j).
    \label{equ:loss-sel}
\end{equation}

\noindent
We add an entropy bonus alongside \cref{equ:loss-sel} to discourage over-committing to a particular primitive early during training. The entropy for the distribution of skill $\iskill$ is:
\begin{equation*}
\mathcal{H}(\pi^{(\iskill)}) = - \sum_{\ilibrary=1}^\nlibrary \pi_\ilibrary^{(\iskill)} \log \pi_\ilibrary^{(\iskill)}.
\end{equation*}
Then, the batch-wide entropy loss $\lent$ is:
\begin{equation}
\lent = -\frac{1}{\nskill N} \sum_{n=1}^N \sum_{\iskill=1}^\nskill \mathcal{H}(\pi^{(\iskill)})^{(n)},
\label{equ:loss-ent}
\end{equation}
where $N$ is the number of samples.

\subsubsection{Parameter Regression Objective}
\label{sec:parameter_objective}
\noindent
\glsunset{dtw}
We use the position difference \gls{sdtw} objective, $\lpos$ for parameter  regression to learn to predict $\vg$ and spatial scale $\ks$ values: 
\begin{equation}
    \lpos(\hat{Y}, Y) = \min \sum_{i,j\in \mathcal{W}} \|\hat{\vy}_i - \vy_j\|_2.
    \label{equ:loss-pos}    
\end{equation}
This objective uses the Euclidean distance on the warping path, and is minimized for a particular selection with appropriate $\vg$ and $\ks$ values. 

The loss in \cref{equ:loss-pos} weighs all points in the trajectory equally; meanwhile, sub-goals are often of higher importance in discrete robot tasks.
Therefore, we employ a terminal objective by penalizing the error at each sub-goal. For our \emph{generated trajectory}, we know the points at the end of each primitive. The \gls{dtw} alignment estimates corresponding points of the reference, which serve as bootstrapped sub-goal targets. We define the terminal position loss as:
\begin{equation}
\lterm = \sum_{\iskill=1}^{\nskill} \left\lVert \hat{\vy}_{T_\iskill} - \vy_{\mathcal{W}(T_\iskill)} \right\rVert,
\end{equation}
where $T_\iskill$ is the terminal step of the $\iskill$-th generated primitive, and $\mathcal{W}(T_\iskill)$ is the corresponding index in the reference. The combined parameter loss, where $\lambda_{\text{pos}}, \lambda_{\text{term}} \in \realp$, is:
\begin{equation}
    \label{eq:loss_param}
    \lpar = \lambda_{\text{pos}} \lpos + \lambda_{\text{term}} \lterm.
\end{equation}

\subsection{Rhythmic Movements}
\label{sec:ext_periodic}
\noindent
We extend \ours to rhythmic movements with minor modifications. Like prior methods~\cite{Yang2022LearningPeriodicTasksa}, we assume repetition information is available; for trajectory data, it can be extracted via \gls{rqa}~\cite{Eckmann1987RecurrencePlotsDynamicala}. To integrate rhythmic primitives into our parallel execution framework, we set the temporal scaling $\tau$ to match the detected number of repetitions, and combine in parallel as in \cref{subsec:combining_dmps}.

The terminal loss $\lterm$ is redefined, as rhythmic movements are characterized by an anchor rather than a goal position; thus, we penalize the distance between the predicted mean position and the mean position of the aligned reference:
\begin{equation}
    \lterm = \left\lVert \left( \frac{1}{T} \sum_{t=1}^T \hat{\vy}_t \right) - \left( \frac{1}{\lvert \mathcal{W}_r \rvert} \sum_{{j \in \mathcal{W}_r}} \vy_j \right) \right\rVert,
\end{equation}
where $\mathcal{W}_r$ is the set of indices in the reference trajectory \gls{dtw}-aligned to the rhythmic execution.

\subsection{Training Curriculum}
\label{sec:practical_implementation}

\noindent
Simultaneously minimizing the selection ($\lsel$) and parameter loss ($\lpar$) complicates optimization. This is because the correct parameters are \emph{specific} to each \gls{dmp}. If the parameter network is optimized using an incorrect primitive, the resulting gradients are uninformative for the true task.  Similarly, inaccurate parameter predictions degrade the invariance properties of $\lsel$, preventing it from providing a stable gradient signal for primitive selection.
\ours addresses this by splitting the training into a three-part curriculum. 

\textbf{Stage 1} trains only the selection head using the warp-invariant objective $\lsel$ (\cref{equ:loss-sel}) with nominal values $ \vg_\ilibrary$, $\ks = 1$. Fixing the parameters ensures that $\lsel$ receives a stable gradient signal. The entropy bonus $\lent$ (\cref{equ:loss-ent}), weighted by $\entweight$, discourages premature commitment to a single primitive; $\entweight$ is annealed over time to allow confident predictions as training progresses.

\textbf{Stage 2} freezes the selection head and optimizes only the parameter heads $ \vg$, $\ks$, ensuring gradients are computed for the correct primitive. The terminal loss is disabled ($\lambda_{\text{term}} = 0$) until parameter predictions are approximately correct.

\textbf{Stage 3} Fine-tunes parameter heads including $\lterm$, which refines sub-goal accuracy.

\subsection{Network Architecture}
\noindent
The task context is the downsampled expert trajectory, augmented with step-wise displacements. We use a fully-convolutional temporal encoder following the literature~\cite{Dwibedi2020CountingOutTime}, with dilated residual blocks sized so the receptive field covers the full input sequence. The encoder output is reduced to $\nskill$ positions, one per predicted primitive, from which separate convolutional heads predict selection probabilities, sub-goals, and spatial scaling. The selection head uses a Softmax activation during training, switching to hard selection after training. An exponential activation is applied to the spatial scaling head's output to ensure $\ks > 0$.

\section{Experiments and Results}
\label{sec:experiments}
\noindent
We evaluate \ours across two distinct domains: a 2D sequential handwriting task and a 3D robot manipulation environment. Through these experiments, we aim to evaluate its ability to generalize across various primitive compositions from unstructured data (\textbf{RQ1}) and perform an ablation study to identify which design choices are most critical to its performance (\textbf{RQ2}).

\begin{table}[t]
\centering
\vspace{8pt}
\caption{{Average DTW distances across varying numbers of movement repetitions, on the test set.}}
\label{tab:dtw_repetition_results}
\begin{tabular}{@{}lccc@{}}
\toprule
& \multicolumn{3}{c}{\textbf{DTW ($\downarrow$)}} \\
\cmidrule(l){2-4}
& \textbf{reps $= 1$} & \textbf{reps $= 2$} & \textbf{reps $= 3$} \\
\midrule
\textbf{BC} & $2.800 \std{6.667}$ & $0.664 \std{0.939}$ & $1.404 \std{0.989}$ \\
\textbf{\baselinemse} & $0.048 \std{0.039}$ & $0.124 \std{0.041}$ & $0.122 \std{0.040}$ \\
\textbf{\baselineplus} & $0.038 \std{0.025}$ & $0.077 \std{0.011}$ & $0.076 \std{0.019}$ \\
\textbf{\ours (ours)} & $\mathbf{0.026} \std{0.016}$ & $\mathbf{0.037} \std{0.014}$ & $\mathbf{0.038} \std{0.008}$ \\
\bottomrule
\end{tabular}
\end{table}

\begin{table*}
\footnotesize
\centering
\vspace{8pt}
\caption{Task success rates across various sub-task sequences in simulation.}
\label{tab:task_success}
\begin{tabular}{@{}llccccc@{}}
\toprule
& & \multicolumn{5}{c}{\textbf{Success Rate (\%) ($\uparrow$)}} \\
\cmidrule(l){3-7}
\textbf{Task} & \textbf{$K$} & \textbf{BC} & \textbf{\baselinemse} & \textbf{\baselineplus} & \textbf{\ours P2 (ours w/o term.)} & \textbf{\ours} (ours) \\
\midrule
\wipe$\rightarrow$\press$\rightarrow$\reach  & 3 & $23.0\std{23.0}$ & $8.0\std{2.0}$ & $57.0\std{21.0}$ & $\mathbf{87.0}\std{21.0}$ & $\mathbf{87.0}\std{21.0}$ \\
\retreat$\rightarrow$\slideLeft$\rightarrow$\reach & 3 & $53.0\std{23.0}$ & $47.0\std{17.0}$ & $57.0\std{22.0}$ & $67.0\std{16.0}$ & $\mathbf{80.0}\std{17.0}$ \\
\circular$\rightarrow$\press$\rightarrow$\slideLeft & 3 & $10.0\std{10.0}$ & $60.0\std{29.0}$ & $76.0\std{16.0}$ & $\mathbf{95.0}\std{10.0}$ & $91.0\std{14.0}$ \\
\slideLeft$\rightarrow$\reach$\rightarrow$\wipe & 3 & $47.0\std{16.0}$ & $54.0\std{16.0}$ & $94.0\std{16.0}$ & $\mathbf{98.0}\std{1.0}$ & $\mathbf{98.0}\std{1.0}$ \\
\reach$\rightarrow$\wipe$\rightarrow$\retreat & 3 & $16.0\std{17.0}$ & $28.0\std{13.0}$ & $91.0\std{6.0}$ & $\mathbf{94.0}\std{3.0}$ & $93.0\std{4.0}$ \\
\circular$\rightarrow$\slideLeft$\rightarrow$\slideRight & 3 & $27.0\std{23.0}$ & $37.0\std{23.0}$ & $\mathbf{97.0}\std{7.0}$ & $\mathbf{97.0}\std{8.0}$ & $\mathbf{97.0}\std{7.0}$ \\
\wipe$\rightarrow$\slideLeft$\rightarrow$\retreat & 3 & $61.0\std{26.0}$ & $62.0\std{17.0}$ & $\mathbf{96.0}\std{2.0}$ & $\mathbf{96.0}\std{3.0}$ & $90.0\std{18.0}$ \\
\slideRight$\rightarrow$\retreat$\rightarrow$\circular & 3 & $0.0\std{0.0}$ & $57.0\std{16.0}$ & $\mathbf{90.0}\std{16.0}$ & $89.0\std{17.0}$ & $89.0\std{17.0}$ \\
\reach$\rightarrow$\wipe$\rightarrow$\press$\rightarrow$\slideLeft & 4 & $30.0\std{9.0}$ & $37.0\std{2.0}$ & $82.0\std{19.0}$ & $\mathbf{96.0}\std{4.0}$ & $\mathbf{96.0}\std{5.0}$ \\
\slideLeft$\rightarrow$\press$\rightarrow$\slideRight$\rightarrow$\retreat & 4 & $40.0\std{14.0}$ & $63.0\std{11.0}$ & $93.0\std{14.0}$ & $\mathbf{100.0}\std{0.0}$ & $\mathbf{100.0}\std{0.0}$ \\
\midrule
\textbf{Overall} & --- & $31.0\std{25.0}$ & $45.0\std{23.0}$ & $83.0\std{21.0}$ & $\mathbf{92.0}\std{13.0}$ & $\mathbf{92.0}\std{14.0}$ \\
\bottomrule
\end{tabular}
\end{table*}

\subsection{Baselines}
\glsunset{dmp}
\noindent
The primary baseline we consider is \textbf{\baselinemse}, based on \cref{subsec:ndmp}, which directly predicts \gls{dmp} weights and goal positions. While this method was designed for single, discrete movements, we extend it to predict $\nskill$ primitives. To ensure a rigorous comparison, we also strengthen \baselinemse by replacing the standard \gls{mse} loss with an \gls{sdtw} loss from \cref{eq:loss_param} and augment the architecture with two heads: one to predict discrete forcing function weights, and one for rhythmic weights, which are combined using \cref{eq:dmp_parallel_combination}. A third head predicts the proportion of each forcing function. We term this variant \baselineplus. As a model-free alternative, we implement \textbf{\gls{bc}} with action-chunking, which maps states directly to velocities. The state includes a phase variable to allow for rhythmic movements. For robot experiments, we also consider \textbf{\ours-P2}, which omits the third phase of training for \ours discussed in \cref{sec:practical_implementation}. All methods are conditioned on the \emph{same task context}.

\subsection{Toy Problem: Sequential Digit-Writing}
\noindent
The first experiment is a digit-writing task inspired by \cite{Li2023ProDMPUnifiedPerspectivea, Bahl2020NeuralDynamicPoliciesa}. While these works considered writing a single discrete digit, we extend this to sequential combinations of pre-defined shapes, including rhythmic movements. We expect a natural level of variation in temporal and spatial scaling, along with the shapes themselves. Specifically, we aim to evaluate whether \ours can generalize across primitive compositions (\textbf{RQ1}) in an unstructured 2D domain.

\subsubsection{Data Collection and Pre-processing}
 The library consists of five shapes: digits 1, 2, 3 (discrete), and a 0 and heart (rhythmic). Sequences use $\nskill=3$ sub-tasks, with 1-3 repetitions of at most one rhythmic movement. Examples are illustrated in \cref{fig:handwriting_plots}. Individual shapes were recorded independently and normalized to 1s at 200Hz to form the library; full sequences were normalized to 3s. We collected 488 demonstrations (approximately 6 demos per sequence), collecting coarse primitive-order annotations used only for evaluation. The task context is the downsampled expert trajectory, arc-length resampled to 150 points. \Gls{rqa} is used to extract repetition counts automatically.

\begin{figure}[t]
    \centering
    \includegraphics[width=0.9\linewidth]{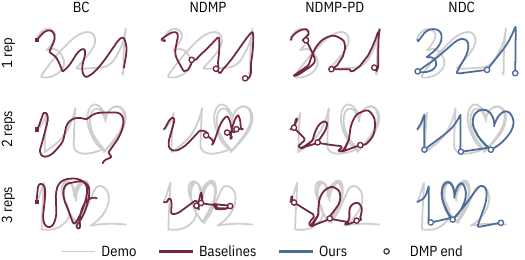}
    \caption{{Rollouts of the digit-writing experiment.}}
    \label{fig:handwriting_plots}
\end{figure}

\subsubsection{Results} 
We group trajectories by number of repetitions and compute the mean and standard deviation of \gls{dtw} distance, normalized by sequence length, between reference and generated trajectories (\cref{tab:dtw_repetition_results}). \Gls{bc} performs worst for long-horizon tasks, as it struggles to learn a policy for different compositions. \baselinemse also performs poorly, especially as repetitions increase. For discrete cases, it moves in the general correct direction, but loses detail as it is susceptible to time-warps. \baselineplus and \ours achieve good performance, with \ours performing best across different numbers of repetitions. \ours also shows the lowest degradation with increasing repetitions. This is consistent with the qualitative examples in \cref{fig:handwriting_plots}, which show \ours performing consistently across different levels of periodicity. Further, \ours is the only method that can accurately segment the task; baselines, even when performant, do not predict meaningful primitives.

\subsection{Robot Simulation Experiments}
\noindent
We consider manipulation tasks in 3D task-space, with randomization in sub-task goals and scaling (\textbf{RQ1}). 
We use the Mecademic Meca 500 in both a Mujoco \cite{Todorov2012MuJoCoPhysicsEnginea} simulation and real-world experiments (Fig.~\ref{fig:sim_real_setup}). 
This setting considers parameters and compositions in order to complete multi-stage tasks in a limited data regime for each task.

\begin{figure}[t!]
    \centering
    \includegraphics[width=0.8\linewidth]{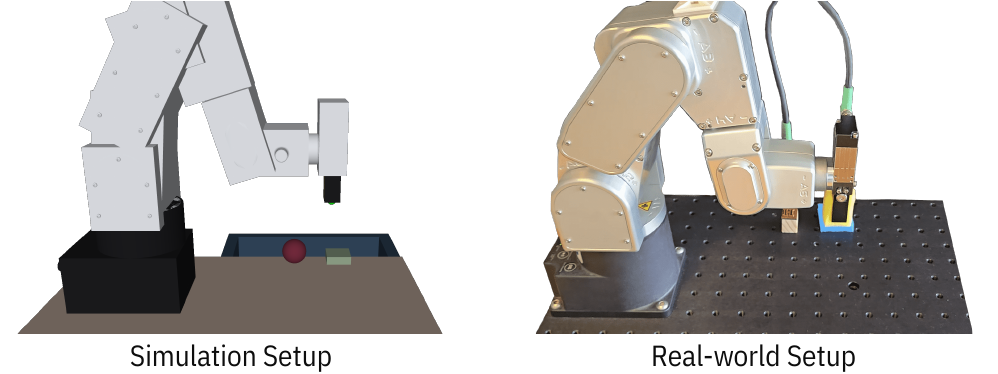}
    \caption{{Simulation and real-world robot setups shown side-by-side.}}
    \label{fig:sim_real_setup}
\end{figure}

\subsubsection{Sub-Task Description and Success}
We consider the following $7$ sub-tasks: \textsf{Reach, Retreat, Press, Slide Left/Right}, and rhythmic \textsf{Circle} (cleaning motion) and \textsf{Wipe} (back and forth). Sub-task success is defined by physical task requirements. \slideLeft~and \slideRight~must place a ball in a designated bin. For \press, a button must be triggered. For \wipe~and \circular, we define ``dirty'' areas divided into cells that clear with contact. \reach~and \retreat~transition between regions and have no success criterion.

We create a tabletop environment with several manipulation tasks composed of $\nskill = 3$ or $4$ sub-tasks described above. 
The environment is randomized by the $xy$-position of objects and scale of cleaning areas. We create 50 scenes for each task, using 40 for training and 10 for evaluation.

\subsubsection{Data Collection}
To collect demonstration data, the target sequence of sub-tasks is executed using heuristic-based trajectory controllers. Positions and scales are adjusted based on the environment's objects. To handle transitions between sequential sub-tasks, we use interpolation to move from the terminal state of one sub-task to the start of the next. We only retain demonstrations that achieve a task success rate of $>95\%$. We further randomize the execution speed between $0.3\times$ and $1.5\times$ the nominal speed. 

\subsubsection{Overall success rate}

\Cref{tab:task_success} summarizes task success rates across different compositions. \Gls{bc} performs worst, followed by \baselinemse which is susceptible to time-warps and periodicity in demonstrations. With a warp-invariant objective and parallel combination to allow periodicity, \baselineplus achieves significantly better results. However, it shows a large variance, falling as low as $57.0\%$. \ours performs well even without the terminal objective, but incorporating it improves the worst-case success from $67.0\%$ to $80.0\%$.

\subsubsection{Sub-task success rate}
The sub-task success rates (\Cref{tab:primitive_success}) further illustrate these trends. \Gls{bc} again performs worst in the demonstration-limited setting; \baselinemse is much better for the \textsf{Slide} tasks, but performs poorly for rhythmic tasks. Both this and \baselineplus also perform very poorly on the \press~sub-task, which demands high-precision. Meanwhile, \ours performs well both with and without the terminal objective (NDC-P2); however this objective improves performance noticeably in the \press~sub-task.

\begin{table}[t]
\footnotesize
\centering
\setlength{\tabcolsep}{2pt}
\vspace{8pt}
\caption{Per-primitive success rates.}
\label{tab:primitive_success}
\begin{tabular}{lrrrrr}
\toprule
& \multicolumn{5}{c}{\textbf{Success Rate (\%) ($\uparrow$)}} \\
\cmidrule(l){2-6}
\textbf{Prim.} & \textbf{BC} & \textbf{\baselinemse} & \textbf{\baselineplus} & \textbf{\ours P2} & \textbf{\ours} \\
\midrule
\slideLeft & $64.0\std{41.0}$ & $80.0\std{28.0}$ & $96.0\std{8.0}$ & $\mathbf{100.0}\std{0.0}$ & $99.0\std{4.0}$ \\
\wipe & $18.0\std{16.0}$ & $21.0\std{10.0}$ & $90.0\std{8.0}$ & $\mathbf{93.0}\std{3.0}$ & $\mathbf{93.0}\std{3.0}$ \\
\press & $12.0\std{8.0}$ & $12.0\std{27.0}$ & $48.0\std{34.0}$ & $86.0\std{17.0}$ & $\mathbf{92.0}\std{11.0}$ \\
\circular & $20.0\std{20.0}$ & $38.0\std{13.0}$ & $\mathbf{93.0}\std{4.0}$ & $88.0\std{9.0}$ & $88.0\std{9.0}$ \\
\slideRight & $10.0\std{17.0}$ & $80.0\std{26.0}$ & $97.0\std{6.0}$ & $\mathbf{100.0}\std{0.0}$ & $\mathbf{100.0}\std{0.0}$ \\
\bottomrule
\end{tabular}
\end{table}

\subsubsection{Goal-prediction accuracy}
We hypothesize that the improvement from \Cref{tab:task_success} is driven by more accurate sub-goal prediction for discrete tasks. To verify this, \cref{tab:goal-error-discrete} reports goal prediction error per method. NDC with $\lterm$ achieves the lowest error, confirming the hypothesis. An example is shown in \cref{fig:goal_preds}.
Interestingly, \baselineplus achieves high success rates despite large goal prediction errors. This suggests it completes sub-tasks through the transient forcing function rather than explicit goal-reaching, a fundamentally different execution strategy that sacrifices interpretability for task completion. Note that \gls{bc} is excluded from this comparison as it does not predict goal positions explicitly.

\begin{table}[t]
\centering
\footnotesize
\setlength{\tabcolsep}{4pt}
\vspace{8pt}
\caption{Goal-prediction error.}
\label{tab:goal-error-discrete}
\begin{tabular}{lcccc}
\toprule
& \multicolumn{4}{c}{\textbf{Goal-Prediction Error (mm) ($\downarrow$)}} \\
\cmidrule(l){2-5}
\textbf{Task} & \textbf{\baselinemse} & \textbf{\baselineplus} & \textbf{\ours P2} & \textbf{\ours} \\
\midrule
\slideLeft & $77.2\std{26.3}$ & $137.9\std{42.4}$ & $12.3\std{8.5}$ & $\mathbf{8.7}\std{5.2}$ \\
\press & $66.0\std{26.0}$ & $67.2\std{28.4}$ & $14.6\std{7.8}$ & $\mathbf{10.5}\std{6.8}$ \\
\slideRight & $114.3\std{20.4}$ & $280.4\std{21.0}$ & $\mathbf{11.6}\std{6.2}$ & $11.9\std{7.3}$ \\
\midrule
\textbf{All} & $80.2\std{30.5}$ & $138.1\std{85.3}$ & $13.1\std{7.8}$ & $\mathbf{10.1}\std{6.3}$ \\
\bottomrule
\end{tabular}
\end{table}

\begin{figure}[t]
    \centering
    \includegraphics[width=0.9\linewidth]{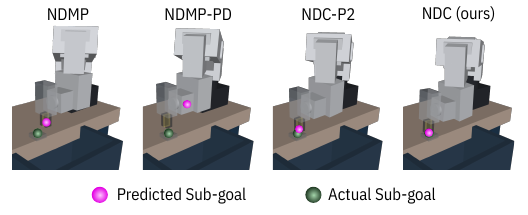}
    \vspace{-1em}
    \caption{Example of predicted v.s. actual goals.}
    \label{fig:goal_preds}    
\end{figure}

\begin{figure}[t]
    \centering
    \includegraphics[width=0.9\linewidth]{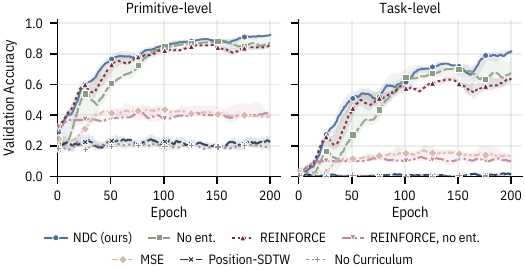}
    \caption{{Ablations of \ours.}}
    \label{fig:ablation_convergence}
    \vspace{-1em}
\end{figure}

\subsection{Ablation Study: Primitive Selection}
\noindent 
We perform an ablation study to analyze the effects of our design choices on \emph{task classification accuracy} (\textbf{RQ2}). 

We compare \ours against the following variations: \textbf{MSE} replaces the \gls{sdtw} similarity metric from \cref{sec:train_objective} with a point-wise metric. \textbf{No Ent.} removes entropy regularization in \cref{equ:loss-ent}. \textbf{REINFORCE} uses a hard-selection policy trained using REINFORCE with baseline~\cite{Sutton1999PolicyGradientMethods}, which we train both with and without entropy regularization. \textbf{Position-SDTW} trains without the selection objective, and \textbf{No Curriculum} trains using the combined objective from the start without the curriculum from \cref{sec:practical_implementation}.

We measure performance using two metrics: \textit{primitive-level accuracy} (percentage of individual primitives correctly identified) and \textit{task-level accuracy} (binary metric requiring all $\nskill$ primitives to be correct) on the handwriting task. Results are shown in \cref{fig:ablation_convergence}. The results indicate that the training curriculum and cosine-similarity objective are the most important levers, followed by warp invariance. The entropy term primarily speeds up convergence in \ours. Meanwhile, for REINFORCE, entropy regularization serves as a necessary exploration mechanism, without which accuracy falls significantly. With entropy, REINFORCE performs well, but falls short of both differentiable variants.

\begin{table*}[t]
\centering
\footnotesize
\vspace{8pt}
\caption{Real World Experiment Results.}
\label{tab:real_world_results}
\begin{tabular}{@{}lcccc@{}}
\toprule
& \multicolumn{4}{c}{\textbf{Success Rate (\%) ($\uparrow$)}} \\
\cmidrule(l){2-5}
\textbf{Task} & \textbf{NDMP} & \textbf{NDMP-PD} & \textbf{NDC P2 (ours w/o term.)}  & \textbf{NDC (ours)} \\
\midrule
\retreat $\rightarrow$ \slideLeft $\rightarrow$  \press & $0.00\std{0.00}$ & $50.00\std{52.70}$ & $90.00\std{31.62}$ & $\mathbf{100.00}\std{0.00}$ \\
\circular $\rightarrow$ \slideLeft $\rightarrow$ \slideRight  & $6.67\std{25.82}$ & $53.33\std{51.64}$ & $66.67\std{48.80}$ & $\mathbf{73.33}\std{45.77}$ \\
\reach $\rightarrow$ \wipe $\rightarrow$ \press $\rightarrow$ \slideLeft & $13.33\std{35.19}$ & $60.00\std{50.71}$ & $86.67\std{35.19}$ & $\mathbf{93.33}\std{25.82}$ \\
\midrule
\textbf{Overall} & $6.67\std{26.67}$ & $54.44\std{50.38}$ & $81.11\std{40.51}$ & $\mathbf{88.89}\std{33.49}$ \\
\bottomrule
\end{tabular}
\end{table*}

\subsection{Real Robot Experiments}
\noindent 
Because this work considers a task-space trajectory as input, the sim-to-real gap in our setting is in robot policy execution. We replicate the test settings from the simulation in the real world for 3 tasks. The resulting physical rollouts are shown in \cref{fig:ndc_robot_rollouts} with more in the supplemental video. We run \baselinemse, \baselineplus and \ours both with (full \ours) and without terminal penalty (\ours-P2) on the same test configurations. 

The results are summarized in \cref{tab:real_world_results}. 
Compared to the simulation, \baselinemse and \baselineplus show a significant drop in success rate, likely due to their sensitivity to the precise trajectory execution in real-world conditions.

\section{Conclusion}
\label{sec:conclusion}
\noindent
This work proposed \ours, which uses \glspl{dmp} to predict \emph{which discrete or rhythmic} primitive to select, as well as its parameters, using suitable optimization objectives that do not require explicit primitive labels. We evaluated \ours on a 2D handwriting task and a tabletop robot simulation task; compared to \baselinemse, \baselineplus, and \gls{bc}, our method maintained the highest performance across most tasks, with the most notable gains under rhythmic repetitions, with ablations showing each component contributes meaningfully. Limitations include a fixed number of primitives per task, assumption of a suitable pre-defined library, and the exclusion of gripper actions. Future work will relax the fixed-primitive-count assumption, potentially via seq2seq approaches, incorporate high-dimensional context modalities such as video, and explore skill discovery to remove the need for a pre-defined primitive library entirely.

\bibliographystyle{IEEEtran}
\bibliography{ref}

\end{document}